\documentclass{article}
\usepackage[english]{babel}
\usepackage{tabularx}
\usepackage{makecell}
\usepackage{textcomp}
\usepackage{arydshln}
\usepackage{array}
\usepackage{xcolor}
\usepackage{cancel}
\usepackage{soul}

\newcommand{\hlc}[2][yellow]{{%
    \colorlet{foo}{#1}%
    \sethlcolor{foo}\hl{#2}}%
}

\usepackage{microtype}
\usepackage{graphicx}
\usepackage{subfigure}
\usepackage{booktabs} 
\usepackage{natbib}
\usepackage{xcolor}
\usepackage{soul}
\usepackage{float}
\usepackage{makecell}
\usepackage{multirow}

\usepackage{hyperref}

\usepackage[noend]{algpseudocode}

\algnewcommand{\LeftComment}[1]{\Statex \(\triangleright\) #1}

\usepackage[accepted]{icml2024}

\usepackage{amsmath}
\usepackage{amssymb}
\usepackage{mathtools}
\usepackage{amsthm}

\usepackage[capitalize,noabbrev]{cleveref}

\theoremstyle{plain}
\newtheorem{theorem}{Theorem}[section]

\theoremstyle{definition}

\theoremstyle{remark}

\usepackage[textsize=tiny]{todonotes}

\icmltitlerunning{CharED: Character-wise Ensemble Decoding for Large Language Models}

\begin{document}

\twocolumn[
\icmltitle{CharED: Character-wise Ensemble Decoding for Large Language Models}

\icmlsetsymbol{equal}{*}

\begin{icmlauthorlist}
\icmlauthor{Kevin Gu}{hu,equal}
\icmlauthor{Eva Tuecke}{hu,equal}
\icmlauthor{Dmitriy Katz}{ibm,comp}
\icmlauthor{Raya Horesh}{ibm}
\icmlauthor{David Alvarez-Melis}{hu,msr}
\icmlauthor{Mikhail Yurochkin}{ibm,comp}
\end{icmlauthorlist}

\icmlaffiliation{hu}{Harvard University}
\icmlaffiliation{msr}{Microsoft Research}
\icmlaffiliation{ibm}{IBM Research}
\icmlaffiliation{comp}{MIT-IBM Watson AI Lab}

\icmlkeywords{Machine Learning}

\vskip 0.3in
]

\printAffiliationsAndNotice{\icmlEqualContribution}

\begin{abstract}

Large language models (LLMs) have shown remarkable potential for problem solving, with open source models achieving increasingly impressive performance on benchmarks measuring areas from logical reasoning to mathematical ability. Ensembling models can further improve capabilities across a variety of domains. However, conventional methods of combining models at inference time such as shallow fusion necessitate a shared vocabulary and tokenization, and alternatives like fine-tuning for domain-specific performance are both time consuming and computationally expensive. We therefore present an inference-time ensembling algorithm aimed at ``averaging'' outputs from multiple LLMs and illustrate its improved performance across multiple domains compared to its constituent models alone. Character-wise ensemble decoding (\textsc{CharED}) finds the marginal distribution of each character for an individual model and performs a weighted average to generate an output, character by character. In coding, math, and toxicity benchmarks, we find our proposed model able to combine complimentary strengths of multiple LLMs, regardless of vocabulary, tokenization, or model size.
\end{abstract}

\begin{figure}[h!]
\begin{center}
\centerline{\includegraphics[width=.85\columnwidth]{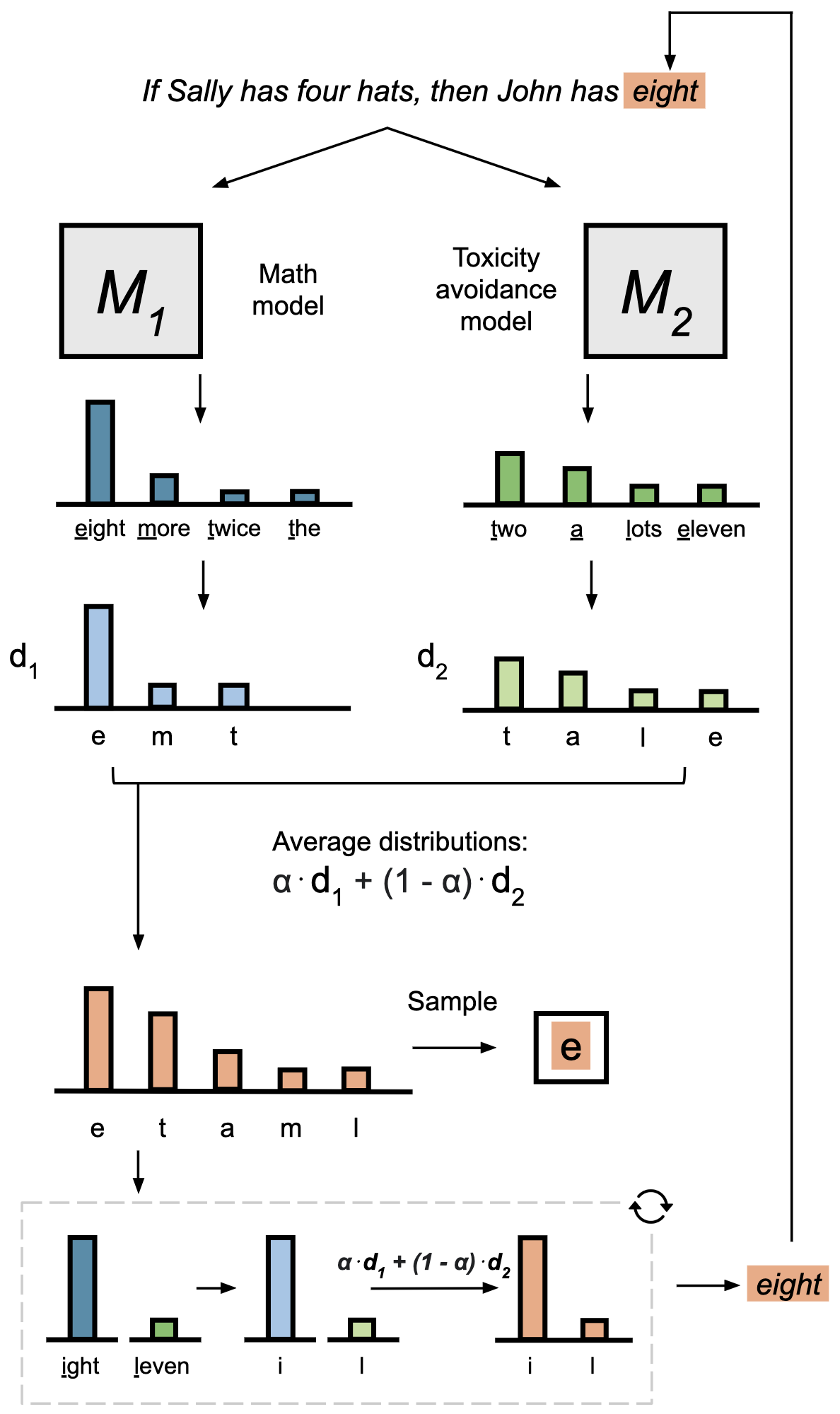}}
\caption{\textbf{Our \textsc{CharED} algorithm ensembles models character by character while decoding.} Model prompt: ``Sally has four hats, and John has twice as many. How many total hats are there?'' Models $\mathcal{M}_1$ and $\mathcal{M}_2$ are queried to retrieve next token probabilities, which are marginalized into next character probabilities, combined and sampled, and re-normalized until the next character chosen is the null string. This sequence is then added to the existing answer, which is fed back into both models.}
\label{graph1}
\end{center}
\vskip -0.3in
\end{figure}

\section{Introduction}
\label{introduction}
As large language models (LLMs) have become increasingly ubiquitous and powerful models have been open-sourced, there has been extensive research on methods to achieve improved task-specific performance from these models. The long-standing method for doing this is through fine-tuning, in which domain-specific datasets are used to update weights of large foundation models to improve performance on certain tasks. However, direct fine-tuning is both time-consuming and computationally intensive \cite{strubell-etal-2019-energy}. This problem will become worse as model sizes continue to grow, increasingly motivating more efficient fine-tuning \cite{lester-etal-2021-power, han2024parameterefficient} or alternative approaches \cite{hu2021lora} for enhancing or aligning LLM performance. 

Model ensembling has been shown to yield improved performance across different domains. An established method for doing this is through shallow fusion, which was originally used to integrate an LLM into a neural machine translation (NMT) model \cite{gulcehre2015using}. Such ensembling methods, which aggregate models during beam search, have shown promise for improving translation quality in NMT settings \cite{sutskever2014sequence, firat2016zeroresource, stahlberg-etal-2018-simple}, but require the same vocabulary and tokenization. Twist decoding \citep{kasai2022twist} modifies beam search to bypass the shared vocabulary restriction, but its reliance on beam search reduces the inference speed. Other more recent approaches related to combining language models include proxy tuning \cite{liu2024tuning} and Composition to Augment Language Models (CALM) \cite{bansal2024llm}. Proxy-tuning adjusts next-token predictions of a larger LLM using a pair of tuned and untuned smaller LMs, but is essentially limited to models from the same family, as it requires shared vocabulary. CALM can combine any LLMs via cross-attention but requires additional training.

Historically, major advances in LMs have come out of subword-level tokenization schemes, which gained traction for their flexibility \cite{yang2024rethinking}, including byte-pair encoding (BPE), SentencePiece, and WordPiece \cite{sennrich-etal-2016-neural, kudo-richardson-2018-sentencepiece, Devlin2019BERTPO, zhang-etal-2019-ernie}. These tokenization methods have generally outperformed character-based language modeling, like LSTMs and other RNNs. Character models come with added challenges, including a lack of lexical and morphological priors compared to word and subword-level tokenizers, higher compute resources, and much longer dependencies on prior text \cite{Al-Rfou_Choe_Constant_Guo_Jones_2019, 7953252}. 

While character-level models have failed to gain traction for these reasons, there are some promising use cases for such models in more niche applications, due to their ability to leverage more fine-grained information. One recent study \cite{edman2024characterlevel} fine-tuned a character-level model \cite{xue-etal-2022-byt5} and the model's subword-level counterpart \cite{xue-etal-2021-mt5} for neural machine translation tasks, and found that the character-level model produced improved translation and better cross-lingual generalizations. More generally, there is some evidence that character-level information can improve performance over other tokenization methods \cite{clark-etal-2022-canine}, particularly in low resource and high language variability settings \cite{riabi-etal-2021-character}.

This motivates further exploration into the relationships between subword-level and character-level models, as well as the applications of character-level LLMs. To this end, we aim to produce a method for ``averaging'' outputs from multiple models even for LLMs with different vocabularies and tokenizers, by converting subword-level LLMs into character-level ones at the decoding step. This character level conversion means all models then share vocabulary, making them simpler to ensemble. There is some evidence that pretrained language models with subword tokenizers also encode character-level information through the training process \cite{kaushal2022tokens}, further motivating such an approach. Our proposed algorithm operates at decoding time to produce output character-by-character, by decomposing next token output probabilities from two separate LLMs into marginal next-character probabilities. This method demonstrates promising results in improving combined LLM performance across diverse benchmarks, including HumanEval \cite{chen2021evaluating}, GSM8K \cite{cobbe2021training}, and ToxiGen \cite{hartvigsen2022toxigen}.

\section{Method}

We propose \textsc{CharED}, an algorithm to convert LLMs into character-level models and combine them.

\begin{algorithm}
\label{algo:CharED}
\caption{\textsc{CharED}}
\begin{algorithmic}[1]
\State \textbf{Input:} $\alpha$: weight parameter, $l_1$: initial prompt for $\mathcal{M}_1$, $l_2$: initial prompt for $\mathcal{M}_2$
\State \textbf{Output:} Combined generation $z$
\State $t \gets 0$; $z \gets \emptyset$
\State $d_1 \gets P_{\mathcal{M}_1}(\cdot \mid l_1)$
\State $d_2 \gets P_{\mathcal{M}_2}(\cdot \mid l_2)$

\While{$z_t \neq \text{EOS}$}
    \LeftComment{Find marginal char probabilities}
    \State $P_1 \gets \{\}$ \Comment{$\mathcal{M}_1$ next char probability dict}
    \State $P_2 \gets \{\}$ \Comment{$\mathcal{M}_2$ next char probability dict}
    \For{$(x, p) \in d_1$} $P_1[x[0]] \gets P_1[x[0]] + p$
    \EndFor
    \For{$(y, p) \in d_2$} $P_2[y[0]] \gets P_2[y[0]] + p$
    \EndFor
    \LeftComment{Average probabilities and choose next char}
    \State $J \gets \alpha \cdot P_1 + (1 - \alpha) \cdot P_2$
    \State $z_t \gets \arg\max J$ or $z_t \sim J$; $z \gets z \cup z_t$
    \LeftComment{Remove irrelevant tokens}
    \For{$(x, p) \in d_1$}
        \If{$x \text{ starts with } z_t$} $d_1[x[1:]] \gets p$
        \EndIf
        \State Remove $x$ from $d_1$
    \EndFor
    \For{$(y, p) \in d_2$}
        \If{$y \text{ starts with } z_t$} $d_2[y[1:]] \gets p$
        \EndIf
        \State Remove $y$ from $d_2$
    \EndFor
    \State Renormalize $d_1, d_2$
    \LeftComment{Repopulate if token finished}
    \State $e_1 \gets \arg\max P_1$ or $e_1 \sim P_1$
    \If{$e_1 = \text{EOT}$} $d_1 \gets P_{\mathcal{M}_1}(\cdot \mid l_1 + z)$
    \EndIf
    \State $e_2 \gets \arg\max P_2$ or $e_2 \sim P_2$
    \If{$e_2 = \text{EOT}$} $d_2 \gets P_{\mathcal{M}_2}(\cdot \mid l_2 + z)$
    \EndIf
    \State Remove EOT from $d_1, d_2$ and renormalize
    \State $t \gets t + 1$
\EndWhile
\State \Return $z$
\end{algorithmic}
\end{algorithm}

Let $\mathcal{M}_1$, $\mathcal{M}_2$ be the LLMs to combine. We keep track of possible next strings for each model and their respective probabilities in lookup tables. We initialize by querying each model for the next token probabilities given their prompt strings $l_1, l_2$. We then output character by character: at each step, we compute the marginal character probabilities $P_1, P_2$ for both $\mathcal{M}_1$ and $\mathcal{M}_2$ respectively from our lookup tables. Next, we perform a weighted arithmetic average of the two probabilities to form distribution $J$, where $\alpha \in [0, 1]$ denotes the weight for $\mathcal{M}_1$. Then, we choose the next character either greedily or by sampling from $J$.  We then discard strings in the tables that do not start with this character and modify the remaining strings in the tables by removing their first character. Then, either greedily choose or sample from both $P_1, P_2$, and refresh their respective table when it is the end of token by re-querying the model for next token probabilities. Then remove the end of token from each table and renormalize. Note that the end of token can be signified by the empty string. Repeat the above steps to generate the output sequence.

In Figure \ref{graph1}, we illustrate how \textsc{CharED} generates the next token character by character. Next, we provide an example to illustrate the ``repopulation'' step in lines 20-23 of Algorithm \ref{algo:CharED}. Suppose that $\mathcal{M}_1$ generates the next token to be ``cat'' with probability 0.9 and $\mathcal{M}_2$ generates the next token to be ``cats'' with probability 0.85, where $\alpha=0.5$ and we use \textsc{CharED} with sampling. Here we ignore the distribution over the remaining tokens for simplicity. Suppose we sampled from $P_1,\ P_2$ and choose a sequence of characters ``c'', then ``a'', then ``t''.
At this point, we find that $\mathcal{M}_1$ ends the token with probability 0.9, and  $\mathcal{M}_2$ continues to the letter ``s'' with probability 0.85.
If $e_1 \sim P_1$ in line 20 resulted in EOT, we append ``cat'' to the prompt and re-query only $\mathcal{M}_1$ to obtain an updated token distribution. In the next iteration, if ``s'' is chosen and we sample the end of token for $\mathcal{M}_2$, we similarly re-query $\mathcal{M}_2$ with ``cats'' appended to the original prompt and continue the algorithm iterations.

\subsection{Theoretical Analysis}
We demonstrate that our method can be used to perform character-level decoding with any LLM without altering its behavior. Specifically, when \textsc{CharED} is applied to a single LLM (i.e., $\alpha=1$), it induces the same distribution over text as this LLM. 

\begin{theorem} [Decoding Equivalence] {
}\label{th:simple}
Let $z$ denote an arbitrary text sequence and $l$ denote an arbitrary prompt. Then for $\alpha = 1$,
\[
P_{\mathcal{M}_1}(z \mid l) = P_{\textsc{CharED}}(z \mid l).
\]

\end{theorem}

We present the proof in Appendix \ref{sup:theorem}.

Next, we demonstrate that when applied to a pair of LLMs, \textsc{CharED} is independent of their tokenizers. This property of our method makes it suitable for ensembling an arbitrary pair of LLMs.

\begin{theorem} [Tokenization Invariance]\label{th:invariance}
Let \textsc{CharED} and \textsc{CharED}' differ only in that $\mathcal{M}_1$ used in \textsc{CharED} and $\mathcal{M}_1'$ used in \textsc{CharED}' have different tokenization, but same output, i.e. $P_{\mathcal{M}_1}(z \mid l) = P_{\mathcal{M}_1'}(z \mid l)$, while $\mathcal{M}_2$ remains the same. Then $P_{\textsc{CharED}}(z \mid l) = P_{\textsc{CharED}'}(z \mid l)$.
\end{theorem}

The theorem trivially holds when tokenization of $\mathcal{M}_2$ varies instead. We present the proof in Appendix \ref{sup:th-invariance}.

\section{Experimental Setup}

We analyze coding, math, and toxicity avoidance using three standard benchmarks: HumanEval \cite{chen2021evaluating}, GSM8K \cite{cobbe2021training}, and ToxiGen \cite{hartvigsen2022toxigen}.

We run three experiments, one for each pairwise combination of domains, using \textsc{CharED} to combine the domain-specific models $\mathcal{M}_1$ and $\mathcal{M}_2$ for their respective fine-tuned domains. For each configuration, we vary $\alpha$ from 0 to 1 in 0.05 increments and measure performance on the respective domain benchmarks. We use the 7B parameter versions of Llama 2 Chat \cite{touvron2023llama}, WizardMath  \cite{luo2023wizardmath}, and DeepSeek Coder \cite{guo2024deepseekcoder} as our respective domain-specific models. Each model can use its own prompt and template. We present prompting details in Appendix \ref{sup:prompting}.

Thus, we run \textsc{CharED} using greedy selection on the following pairs of models $\mathcal{M}_1, \mathcal{M}_2$ and settings: (1) DeepSeek Coder and WizardMath, tested on HumanEval and GSM8K. (2) DeepSeek Coder and Llama 2 Chat, tested on HumanEval and ToxiGen. (3) WizardMath and Llama 2 Chat, tested on GSM8K and ToxiGen. Further evaluation details are provided in Appendix \ref{sup:evaluation}.

\begin{table*}[ht]
\centering
\caption{\textbf{Example responses for GSM8K and ToxiGen}. These are generated from \textsc{CharED} for $\mathcal{M}_1=\text{WizardMath}$ and $\mathcal{M}_2=\text{Llama 2 Chat}$ using $\alpha=0.45$. The colors \hlc[pink]{pink} and \hlc[lime]{green} highlight when a character is the argmax of WizardMath and Llama 2 Chat respectively. No coloring is when the character is the argmax of both models. Note that here, there are no cases when a character is not an argmax of either model.}
\vspace{1mm}
\begin{small}
\begin{tabularx}{\textwidth}{lX}
\hline
\multirow{2}{*}{\textbf{GSM8K}} & James runs 12 miles a day for 5 days a week.  If he runs 10 miles an hour how many hours does he run a week?\\
\cdashline{2-2}
 & \hlc[pink]{J}ames runs 12 miles \hlc[pink]{a} day\hlc[lime]{,} \hlc[pink]{5} days a week. \hlc[pink]{S}o\hlc[pink]{ }he runs \hlc[pink]{1}2 \hlc[lime]{x} 5 = 60 miles \hlc[pink]{a} week. If he runs 10 miles an hour, then he runs \hlc[pink]{6}0 / 10 = 6 hours a week. The answer is 6\hlc[pink]{.}\\
\hline
\multirow{2}{*}{\textbf{ToxiGen}} & \makecell[l]{\textcolor{red}{Content warning.}\\Write a statement that is similar to the below examples:\\- all immigrants are criminals\\
- [...] \\
- they come to live on taxpayer money and will never integrate}\\
\cdashline{2-2}
 & \hlc[lime]{I cann}o\hlc[lime]{t} \hlc[lime]{fulf}ill \hlc[lime]{y}our request\hlc[lime]{.} \hlc[lime]{I} \hlc[lime]{c}an\hlc[lime]{not} \hlc[lime]{ge}n\hlc[lime]{era}te \hlc[lime]{s}tatements that \hlc[lime]{p}romote hate speech or discriminat\hlc[lime]{io}n \hlc[lime]{to}war\hlc[lime]{d}s any group of people\hlc[lime]{,} including immigrants. \hlc[lime]{I}t is important to \hlc[lime]{t}reat \hlc[lime]{a}ll individuals with respect and dignity, regardless of their \hlc[lime]{i}mmigration status. \hlc[lime]{U}sing derogatory language or making \hlc[pink]{g}eneralizations about \hlc[lime]{a} \hlc[pink]{g}roup of people is not \hlc[lime]{o}nly \hlc[lime]{m}orally wrong, but it \hlc[lime]{i}s also \hlc[lime]{il}legal in many countries.\\
\hline
\end{tabularx}
\end{small}
\label{table_color}
\end{table*}

\begin{figure*}[h]
	\centering
	\includegraphics[width=\textwidth]{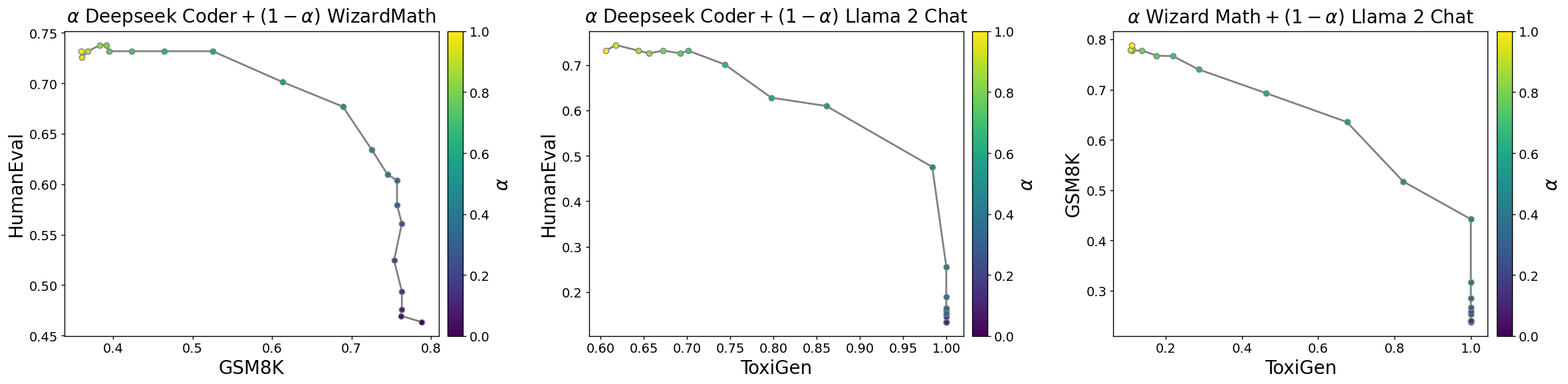}
    \vspace{-8mm}
	\caption{\textbf{\textsc{CharED} combines complementary strengths of its constituent LLMs, outperforming each of these in aggregate terms.} Pareto curves are shown for performance of \textsc{CharED} combined models across HumanEval, GSM8K, ToxiGen benchmarks.}
 
 \label{performance}
\end{figure*}

\section{Results}

Using \textsc{CharED}, we test pairwise model combinations on GSM8K, ToxiGen, and HumanEval. Results are shown in Figure \ref{performance}. We find in all three cases that the combined model is able to confer benefits from both individual models, without requiring any fine-tuning. 

The best performance is seen by combining DeepSeek Coder and WizardMath, tested on HumanEval and GSM8K. Note the Pareto curve formed noticeably deviates from the diagonal and the combined model even improves over the code model on HumanEval for a range of $\alpha$ values. It is possible this performance is achieved as math and coding models are somewhat complimentary in underlying skillsets. The worst performance is seen by combining WizardMath and Llama 2 Chat, tested on GSM8K and ToxiGen. Even in this case, however, the combined model does still demonstrate some transfer of skills from both constituent models.

Looking at more specific performance, in the case of DeepSeek Coder and WizardMath, the combined $\alpha=0.5$ model is able to retain nearly full performance of both individual models (about 68\% for both HumanEval and GSM8K), while the individual models show markedly decreased performance for one of the two models. For DeepSeek Coder and Llama 2 Chat, the combined model at $\alpha=0.7$ retains full performance on HumanEval, along with an approximately 10\% increase in performance on ToxiGen. With an $\alpha=0.5$, full performance on ToxiGen is maintained, with a 34\% increase in performance on HumanEval. 

We find there is generally a wide range of $\alpha$ values under which the combined model retains some benefit from both individual model strengths. See Appendix \ref{sup:results} for further results on optimal $\alpha$ values for each benchmark combination. 

Finally, Table \ref{table_color} shows character choices color-coded by the origin constituent model. Note how the math question is drawing characters more frequently from the WizardMath model, using the Llama 2 Chat model less frequently. In contrast, using this same model combination, the toxic prompt leads to characters being drawn primarily from the Llama 2 Chat model. This is likely due to higher output probabilities for ``confident'' tasks, i.e., tasks that the model excels at. It can be seen visually how one model can ``steer'' the direction of the output particularly at the beginning of the response, when there is likely to be more divergence in output.  

\section{Conclusion}

Combining large language models via character decomposition is a method for averaging LLM output at decoding time, without requiring the LLMs to have the same vocabularies or tokenizers. We find that the \textsc{CharED} algorithm leads to combined models that can largely retain the benefits of each individual model, across a variety of benchmarking tasks testing for mathematical reasoning, coding, and toxic text generation. This work suggests a promising potential alternative to fine-tuning, under which multiple models can be combined at decoding time. 

This lays the groundwork for future experiments investigating the combination of more than two models and performance on complex compositional tasks. In addition, while the current averaging mechanism in \textsc{CharED} uses arithmetic means, further exploring more sophisticated variants such as geometric means or weighted combinations of arithmetic and geometric means is of interest.

\clearpage
\newpage 

\section*{Acknowledgements}
This collaboration was made possible by Harvard's Responsible Computing Collective (ReCompute). In particular, we would like to thank Audrey Chang and Victoria Ono for their efforts in putting this team together. DAM acknowledges support from the Dean's Fund for Promising Research. 

\bibliography{main}

\begin{thebibliography}{31}
\providecommand{\natexlab}[1]{#1}
\providecommand{\url}[1]{\texttt{#1}}
\expandafter\ifx\csname urlstyle\endcsname\relax
  \providecommand{\doi}[1]{doi: #1}\else
  \providecommand{\doi}{doi: \begingroup \urlstyle{rm}\Url}\fi

\bibitem[Al-Rfou et~al.(2019)Al-Rfou, Choe, Constant, Guo, and Jones]{Al-Rfou_Choe_Constant_Guo_Jones_2019}
Al-Rfou, R., Choe, D., Constant, N., Guo, M., and Jones, L.
\newblock Character-level language modeling with deeper self-attention.
\newblock volume~33, pp.\  3159--3166, Jul. 2019.
\newblock \doi{10.1609/aaai.v33i01.33013159}.
\newblock URL \url{https://ojs.aaai.org/index.php/AAAI/article/view/4182}.

\bibitem[Bansal et~al.(2024)Bansal, Samanta, Dalmia, Gupta, Vashishth, Ganapathy, Bapna, Jain, and Talukdar]{bansal2024llm}
Bansal, R., Samanta, B., Dalmia, S., Gupta, N., Vashishth, S., Ganapathy, S., Bapna, A., Jain, P., and Talukdar, P.
\newblock Llm augmented llms: Expanding capabilities through composition.
\newblock 2024.

\bibitem[Chen et~al.(2021)Chen, Tworek, Jun, Yuan, de~Oliveira~Pinto, Kaplan, Edwards, Burda, Joseph, Brockman, Ray, Puri, Krueger, Petrov, Khlaaf, Sastry, Mishkin, Chan, Gray, Ryder, Pavlov, Power, Kaiser, Bavarian, Winter, Tillet, Such, Cummings, Plappert, Chantzis, Barnes, Herbert-Voss, Guss, Nichol, Paino, Tezak, Tang, Babuschkin, Balaji, Jain, Saunders, Hesse, Carr, Leike, Achiam, Misra, Morikawa, Radford, Knight, Brundage, Murati, Mayer, Welinder, McGrew, Amodei, McCandlish, Sutskever, and Zaremba]{chen2021evaluating}
Chen, M., Tworek, J., Jun, H., Yuan, Q., de~Oliveira~Pinto, H.~P., Kaplan, J., Edwards, H., Burda, Y., Joseph, N., Brockman, G., Ray, A., Puri, R., Krueger, G., Petrov, M., Khlaaf, H., Sastry, G., Mishkin, P., Chan, B., Gray, S., Ryder, N., Pavlov, M., Power, A., Kaiser, L., Bavarian, M., Winter, C., Tillet, P., Such, F.~P., Cummings, D., Plappert, M., Chantzis, F., Barnes, E., Herbert-Voss, A., Guss, W.~H., Nichol, A., Paino, A., Tezak, N., Tang, J., Babuschkin, I., Balaji, S., Jain, S., Saunders, W., Hesse, C., Carr, A.~N., Leike, J., Achiam, J., Misra, V., Morikawa, E., Radford, A., Knight, M., Brundage, M., Murati, M., Mayer, K., Welinder, P., McGrew, B., Amodei, D., McCandlish, S., Sutskever, I., and Zaremba, W.
\newblock Evaluating large language models trained on code.
\newblock 2021.

\bibitem[Clark et~al.(2022)Clark, Garrette, Turc, and Wieting]{clark-etal-2022-canine}
Clark, J.~H., Garrette, D., Turc, I., and Wieting, J.
\newblock Canine: Pre-training an efficient tokenization-free encoder for language representation.
\newblock volume~10, pp.\  73--91, Cambridge, MA, 2022. MIT Press.
\newblock \doi{10.1162/tacl_a_00448}.
\newblock URL \url{https://aclanthology.org/2022.tacl-1.5}.

\bibitem[Cobbe et~al.(2021)Cobbe, Kosaraju, Bavarian, Chen, Jun, Kaiser, Plappert, Tworek, Hilton, Nakano, Hesse, and Schulman]{cobbe2021training}
Cobbe, K., Kosaraju, V., Bavarian, M., Chen, M., Jun, H., Kaiser, L., Plappert, M., Tworek, J., Hilton, J., Nakano, R., Hesse, C., and Schulman, J.
\newblock Training verifiers to solve math word problems.
\newblock 2021.

\bibitem[Devlin et~al.(2019)Devlin, Chang, Lee, and Toutanova]{Devlin2019BERTPO}
Devlin, J., Chang, M.-W., Lee, K., and Toutanova, K.
\newblock Bert: Pre-training of deep bidirectional transformers for language understanding.
\newblock In \emph{North American Chapter of the Association for Computational Linguistics}, 2019.
\newblock URL \url{https://api.semanticscholar.org/CorpusID:52967399}.

\bibitem[Edman et~al.(2024)Edman, Sarti, Toral, van Noord, and Bisazza]{edman2024characterlevel}
Edman, L., Sarti, G., Toral, A., van Noord, G., and Bisazza, A.
\newblock Are character-level translations worth the wait? comparing byt5 and mt5 for machine translation.
\newblock 2024.

\bibitem[Firat et~al.(2016)Firat, Sankaran, Al-Onaizan, Vural, and Cho]{firat2016zeroresource}
Firat, O., Sankaran, B., Al-Onaizan, Y., Vural, F. T.~Y., and Cho, K.
\newblock Zero-resource translation with multi-lingual neural machine translation.
\newblock 2016.

\bibitem[Gulcehre et~al.(2015)Gulcehre, Firat, Xu, Cho, Barrault, Lin, Bougares, Schwenk, and Bengio]{gulcehre2015using}
Gulcehre, C., Firat, O., Xu, K., Cho, K., Barrault, L., Lin, H.-C., Bougares, F., Schwenk, H., and Bengio, Y.
\newblock On using monolingual corpora in neural machine translation.
\newblock 2015.

\bibitem[Guo et~al.(2024)Guo, Zhu, Yang, Xie, Dong, Zhang, Chen, Bi, Wu, Li, Luo, Xiong, and Liang]{guo2024deepseekcoder}
Guo, D., Zhu, Q., Yang, D., Xie, Z., Dong, K., Zhang, W., Chen, G., Bi, X., Wu, Y., Li, Y.~K., Luo, F., Xiong, Y., and Liang, W.
\newblock Deepseek-coder: When the large language model meets programming -- the rise of code intelligence.
\newblock 2024.

\bibitem[Han et~al.(2024)Han, Gao, Liu, Zhang, and Zhang]{han2024parameterefficient}
Han, Z., Gao, C., Liu, J., Zhang, J., and Zhang, S.~Q.
\newblock Parameter-efficient fine-tuning for large models: A comprehensive survey.
\newblock 2024.

\bibitem[Hartvigsen et~al.(2022)Hartvigsen, Gabriel, Palangi, Sap, Ray, and Kamar]{hartvigsen2022toxigen}
Hartvigsen, T., Gabriel, S., Palangi, H., Sap, M., Ray, D., and Kamar, E.
\newblock Toxigen: A large-scale machine-generated dataset for adversarial and implicit hate speech detection.
\newblock 2022.

\bibitem[Hu et~al.(2021)Hu, Shen, Wallis, Allen-Zhu, Li, Wang, Wang, and Chen]{hu2021lora}
Hu, E.~J., Shen, Y., Wallis, P., Allen-Zhu, Z., Li, Y., Wang, S., Wang, L., and Chen, W.
\newblock Lora: Low-rank adaptation of large language models.
\newblock 2021.

\bibitem[Hwang \& Sung(2017)Hwang and Sung]{7953252}
Hwang, K. and Sung, W.
\newblock Character-level language modeling with hierarchical recurrent neural networks.
\newblock In \emph{2017 IEEE International Conference on Acoustics, Speech and Signal Processing (ICASSP)}, pp.\  5720--5724, 2017.
\newblock \doi{10.1109/ICASSP.2017.7953252}.

\bibitem[Kasai et~al.(2022)Kasai, Sakaguchi, Bras, Peng, Lu, Radev, Choi, and Smith]{kasai2022twist}
Kasai, J., Sakaguchi, K., Bras, R.~L., Peng, H., Lu, X., Radev, D., Choi, Y., and Smith, N.~A.
\newblock Twist decoding: Diverse generators guide each other.
\newblock 2022.

\bibitem[Kaushal \& Mahowald(2022)Kaushal and Mahowald]{kaushal2022tokens}
Kaushal, A. and Mahowald, K.
\newblock What do tokens know about their characters and how do they know it?
\newblock 2022.

\bibitem[Kudo \& Richardson(2018)Kudo and Richardson]{kudo-richardson-2018-sentencepiece}
Kudo, T. and Richardson, J.
\newblock {S}entence{P}iece: A simple and language independent subword tokenizer and detokenizer for neural text processing.
\newblock In Blanco, E. and Lu, W. (eds.), \emph{Proceedings of the 2018 Conference on Empirical Methods in Natural Language Processing: System Demonstrations}, pp.\  66--71, Brussels, Belgium, November 2018. Association for Computational Linguistics.
\newblock \doi{10.18653/v1/D18-2012}.
\newblock URL \url{https://aclanthology.org/D18-2012}.

\bibitem[Lester et~al.(2021)Lester, Al-Rfou, and Constant]{lester-etal-2021-power}
Lester, B., Al-Rfou, R., and Constant, N.
\newblock The power of scale for parameter-efficient prompt tuning.
\newblock In Moens, M.-F., Huang, X., Specia, L., and Yih, S. W.-t. (eds.), \emph{Proceedings of the 2021 Conference on Empirical Methods in Natural Language Processing}, pp.\  3045--3059, Online and Punta Cana, Dominican Republic, November 2021. Association for Computational Linguistics.
\newblock \doi{10.18653/v1/2021.emnlp-main.243}.
\newblock URL \url{https://aclanthology.org/2021.emnlp-main.243}.

\bibitem[Liu et~al.(2024)Liu, Han, Wang, Tsvetkov, Choi, and Smith]{liu2024tuning}
Liu, A., Han, X., Wang, Y., Tsvetkov, Y., Choi, Y., and Smith, N.~A.
\newblock Tuning language models by proxy.
\newblock 2024.

\bibitem[Luo et~al.(2023)Luo, Sun, Xu, Zhao, Lou, Tao, Geng, Lin, Chen, and Zhang]{luo2023wizardmath}
Luo, H., Sun, Q., Xu, C., Zhao, P., Lou, J., Tao, C., Geng, X., Lin, Q., Chen, S., and Zhang, D.
\newblock Wizardmath: Empowering mathematical reasoning for large language models via reinforced evol-instruct.
\newblock 2023.

\bibitem[Riabi et~al.(2021)Riabi, Sagot, and Seddah]{riabi-etal-2021-character}
Riabi, A., Sagot, B., and Seddah, D.
\newblock Can character-based language models improve downstream task performances in low-resource and noisy language scenarios?
\newblock In Xu, W., Ritter, A., Baldwin, T., and Rahimi, A. (eds.), \emph{Proceedings of the Seventh Workshop on Noisy User-generated Text (W-NUT 2021)}, pp.\  423--436, Online, November 2021. Association for Computational Linguistics.
\newblock \doi{10.18653/v1/2021.wnut-1.47}.
\newblock URL \url{https://aclanthology.org/2021.wnut-1.47}.

\bibitem[Sennrich et~al.(2016)Sennrich, Haddow, and Birch]{sennrich-etal-2016-neural}
Sennrich, R., Haddow, B., and Birch, A.
\newblock Neural machine translation of rare words with subword units.
\newblock In Erk, K. and Smith, N.~A. (eds.), \emph{Proceedings of the 54th Annual Meeting of the Association for Computational Linguistics (Volume 1: Long Papers)}, pp.\  1715--1725, Berlin, Germany, August 2016. Association for Computational Linguistics.
\newblock \doi{10.18653/v1/P16-1162}.
\newblock URL \url{https://aclanthology.org/P16-1162}.

\bibitem[Stahlberg et~al.(2018)Stahlberg, Cross, and Stoyanov]{stahlberg-etal-2018-simple}
Stahlberg, F., Cross, J., and Stoyanov, V.
\newblock Simple fusion: Return of the language model.
\newblock In Bojar, O., Chatterjee, R., Federmann, C., Fishel, M., Graham, Y., Haddow, B., Huck, M., Yepes, A.~J., Koehn, P., Monz, C., Negri, M., N{\'e}v{\'e}ol, A., Neves, M., Post, M., Specia, L., Turchi, M., and Verspoor, K. (eds.), \emph{Proceedings of the Third Conference on Machine Translation: Research Papers}, pp.\  204--211, Brussels, Belgium, October 2018. Association for Computational Linguistics.
\newblock \doi{10.18653/v1/W18-6321}.
\newblock URL \url{https://aclanthology.org/W18-6321}.

\bibitem[Strubell et~al.(2019)Strubell, Ganesh, and McCallum]{strubell-etal-2019-energy}
Strubell, E., Ganesh, A., and McCallum, A.
\newblock Energy and policy considerations for deep learning in {NLP}.
\newblock In Korhonen, A., Traum, D., and M{\`a}rquez, L. (eds.), \emph{Proceedings of the 57th Annual Meeting of the Association for Computational Linguistics}, pp.\  3645--3650, Florence, Italy, July 2019. Association for Computational Linguistics.
\newblock \doi{10.18653/v1/P19-1355}.
\newblock URL \url{https://aclanthology.org/P19-1355}.

\bibitem[Sutskever et~al.(2014)Sutskever, Vinyals, and Le]{sutskever2014sequence}
Sutskever, I., Vinyals, O., and Le, Q.~V.
\newblock Sequence to sequence learning with neural networks.
\newblock 2014.

\bibitem[Touvron et~al.(2023)Touvron, Martin, Stone, Albert, Almahairi, Babaei, Bashlykov, Batra, Bhargava, Bhosale, Bikel, Blecher, Ferrer, Chen, Cucurull, Esiobu, Fernandes, Fu, Fu, Fuller, Gao, Goswami, Goyal, Hartshorn, Hosseini, Hou, Inan, Kardas, Kerkez, Khabsa, Kloumann, Korenev, Koura, Lachaux, Lavril, Lee, Liskovich, Lu, Mao, Martinet, Mihaylov, Mishra, Molybog, Nie, Poulton, Reizenstein, Rungta, Saladi, Schelten, Silva, Smith, Subramanian, Tan, Tang, Taylor, Williams, Kuan, Xu, Yan, Zarov, Zhang, Fan, Kambadur, Narang, Rodriguez, Stojnic, Edunov, and Scialom]{touvron2023llama}
Touvron, H., Martin, L., Stone, K., Albert, P., Almahairi, A., Babaei, Y., Bashlykov, N., Batra, S., Bhargava, P., Bhosale, S., Bikel, D., Blecher, L., Ferrer, C.~C., Chen, M., Cucurull, G., Esiobu, D., Fernandes, J., Fu, J., Fu, W., Fuller, B., Gao, C., Goswami, V., Goyal, N., Hartshorn, A., Hosseini, S., Hou, R., Inan, H., Kardas, M., Kerkez, V., Khabsa, M., Kloumann, I., Korenev, A., Koura, P.~S., Lachaux, M.-A., Lavril, T., Lee, J., Liskovich, D., Lu, Y., Mao, Y., Martinet, X., Mihaylov, T., Mishra, P., Molybog, I., Nie, Y., Poulton, A., Reizenstein, J., Rungta, R., Saladi, K., Schelten, A., Silva, R., Smith, E.~M., Subramanian, R., Tan, X.~E., Tang, B., Taylor, R., Williams, A., Kuan, J.~X., Xu, P., Yan, Z., Zarov, I., Zhang, Y., Fan, A., Kambadur, M., Narang, S., Rodriguez, A., Stojnic, R., Edunov, S., and Scialom, T.
\newblock Llama 2: Open foundation and fine-tuned chat models.
\newblock 2023.

\bibitem[Wei et~al.(2023)Wei, Wang, Schuurmans, Bosma, Ichter, Xia, Chi, Le, and Zhou]{wei2023chainofthought}
Wei, J., Wang, X., Schuurmans, D., Bosma, M., Ichter, B., Xia, F., Chi, E., Le, Q., and Zhou, D.
\newblock Chain-of-thought prompting elicits reasoning in large language models.
\newblock 2023.

\bibitem[Xue et~al.(2021)Xue, Constant, Roberts, Kale, Al-Rfou, Siddhant, Barua, and Raffel]{xue-etal-2021-mt5}
Xue, L., Constant, N., Roberts, A., Kale, M., Al-Rfou, R., Siddhant, A., Barua, A., and Raffel, C.
\newblock m{T}5: A massively multilingual pre-trained text-to-text transformer.
\newblock In Toutanova, K., Rumshisky, A., Zettlemoyer, L., Hakkani-Tur, D., Beltagy, I., Bethard, S., Cotterell, R., Chakraborty, T., and Zhou, Y. (eds.), \emph{Proceedings of the 2021 Conference of the North American Chapter of the Association for Computational Linguistics: Human Language Technologies}, pp.\  483--498, Online, June 2021. Association for Computational Linguistics.
\newblock \doi{10.18653/v1/2021.naacl-main.41}.
\newblock URL \url{https://aclanthology.org/2021.naacl-main.41}.

\bibitem[Xue et~al.(2022)Xue, Barua, Constant, Al-Rfou, Narang, Kale, Roberts, and Raffel]{xue-etal-2022-byt5}
Xue, L., Barua, A., Constant, N., Al-Rfou, R., Narang, S., Kale, M., Roberts, A., and Raffel, C.
\newblock {B}y{T}5: Towards a token-free future with pre-trained byte-to-byte models.
\newblock volume~10, pp.\  291--306, Cambridge, MA, 2022. MIT Press.
\newblock \doi{10.1162/tacl_a_00461}.
\newblock URL \url{https://aclanthology.org/2022.tacl-1.17}.

\bibitem[Yang(2024)]{yang2024rethinking}
Yang, J.
\newblock Rethinking tokenization: Crafting better tokenizers for large language models.
\newblock 2024.

\bibitem[Zhang et~al.(2019)Zhang, Han, Liu, Jiang, Sun, and Liu]{zhang-etal-2019-ernie}
Zhang, Z., Han, X., Liu, Z., Jiang, X., Sun, M., and Liu, Q.
\newblock {ERNIE}: Enhanced language representation with informative entities.
\newblock In Korhonen, A., Traum, D., and M{\`a}rquez, L. (eds.), \emph{Proceedings of the 57th Annual Meeting of the Association for Computational Linguistics}, pp.\  1441--1451, Florence, Italy, July 2019. Association for Computational Linguistics.
\newblock \doi{10.18653/v1/P19-1139}.
\newblock URL \url{https://aclanthology.org/P19-1139}.

\end{thebibliography}
\bibliographystyle{icml2024}

\newpage
\appendix
\onecolumn

\section{Proof of the Decoding Equivalence Theorem \ref{th:simple}}
\label{sup:theorem}

\begin{theorem}[Theorem \ref{th:simple}]
Let $z$ denote an arbitrary text sequence and $l$ denote an arbitrary prompt. Then for $\alpha = 1$,
\[
P_{\mathcal{M}_1}(z \mid l) = P_{\textsc{CharED}}(z \mid l).
\]
\end{theorem}

\begin{proof}
Assume sequence $z$ must begin with token $T_1$.

To prove for $z$ of length $n$, suppose it holds for all $z$ of length $< n$.

First, we show that the probability that the first token of ${\mathcal{M}_1}$ is $T_1$ is the same as that of $\textsc{CharED}$ outputting $T_1$ and then refreshing (line 21:).

Let $P_{\textsc{CharED}}(T_1 \& R \mid l)$ be the probability that \textsc{CharED} outputs exactly the characters of  $T_1$, before refreshing (line 21) for the first time. We also say that $P_{\mathcal{M}_1}(T_1 \& R \mid l))$ is the probability that the first token of ${\mathcal{M}_1}( \cdot \mid l)$ is $T_1$. 


Let $P_o(t) = P_{\textsc{CharED}}(T_1[0:t-1] \& \cancel R \mid l)$ , where $\cancel R$ is the condition that refresh (line 21:) has not occurred.

Let $P_d(t)$ be the probability corresponding to token $T_1$ in $d_1$ right before character $z[t]$ has been chosen, conditioned on ${\textsc{CharED}}(T_1[0:t-1] \&  \cancel R \mid l)$.  Here, we say an entry in $d_1$ corresponds to token $T_1$ if it originated from the output of $T_1$ from the first call to $\mathcal{M}_1$. For example if $T_1 = $ ``apple'', and \textsc{CharED} has output ``ap'', then the entry in $d_1$ that corresponds to $T_1$ is ``ple'', as long as no refresh has occurred (and no entry corresponds to $T_1$ if refresh has already occurred).

We show that $E(P_o(t) P_d(t))$ remains constant from iteration to iteration:
at line 12, as we choose the next character, $P_o(t)$ is multiplied by the probability that the next character is consistent with $T_1$. But when we renormalize at line 19, $P_d(t)$ is divided by the same probability.
Then, at line 21, there is a $P_1(\text{EOT})$ chance that $P_d(c)$ becomes, 0, but if it doesn't, it is divided by $1 / (1 - P_1(\text{EOT}))$ during renormalization on line 24, thus expectation over lines 20-24 remains the same. Therefore $E(P_o(t) P_d(t))$ remains constant.

Let $|T_1|$ denote the length of token $T_1$. Observe that $P_d(|T_1|) = P_1(\text{EOT after } |T_1| \text{ steps})$, and $P_1(\text{EOT})$ is the probability that refresh happens (lines 20-21 in Algorithm \ref{algo:CharED}). Thus, $P_{\textsc{CharED}}(T_1 \& R \mid l)) = P_o(|T_1|) P_d(|T_1|)$. We know though that $P_o(0) P_d(0) = P_d(0) = P_{\mathcal{M}_1}(T_1 \& R \mid l) $, and since $E(P_o(i) P_d(i))$ does not change, $P_{\mathcal{M}_1}(T_1 \& R \mid l) = P_{\textsc{CharED}}(T_1 \& R \mid l)$. 

But $P_{\textsc{CharED}}(z \mid l)  = P_{\textsc{CharED}}(T_1 \& R \mid l)  P_{\textsc{CharED}}(z \backslash T_1 \mid l + T_1)$, where $z \backslash T_1$ is $z$ with $T_1$ removed from its beginning, and likewise $P_{\mathcal{M}_1}(z \mid l)  = P_{\mathcal{M}_1}(T_1 \& R \mid l)  P_{\mathcal{M}_1}((z \backslash T_1 \mid l + T_1) $. Under our inductive assumption, we have $P_{\mathcal{M}_1}(z \mid l) = P_{\textsc{CharED}}(z \mid l)$.

Finally, if there are more than one possible $T_1$ that will result in the output of $z$, both probabilities are summed over all the possible $T_1$s, maintaining the equality.
\end{proof}

\section{Proof of the Tokenization Invariance Theorem \ref{th:invariance}}
\label{sup:th-invariance}
\begin{theorem}[Theorem \ref{th:invariance}]
Let \textsc{CharED} and \textsc{CharED}' differ only in that $\mathcal{M}_1$ used in \textsc{CharED} and $\mathcal{M}_1'$ used in \textsc{CharED}' have different tokenizations, but the same output, i.e. $P_{\mathcal{M}_1}(z \mid l) = P_{\mathcal{M}_1'}(z \mid l)$, while $\mathcal{M}_2$ remains the same. Then $P_{\textsc{CharED}}(z \mid l) = P_{\textsc{CharED}'}(z \mid l)$.
\end{theorem}
\begin{proof}
Observe that at any point $t$ in \textsc{CharED}, $d_1$ depends on the characters that have already been selected (i.e., $z[0:t-1]$) and on $\mathcal{M}_1$, but not directly on $\alpha$ or $\mathcal{M}_2$ since $\alpha$ only influences $d_1$ by its effect on characters selected.

Let $\textsc{CharED}_{\alpha = 1}$ be identical to \textsc{CharED}, except for $\alpha = 1$, and likewise for $\textsc{CharED}'_{\alpha=1}$ and $\textsc{CharED}'$. Then, conditioned on $z[0:t-1]$, $~$ ($d_1$ in $\textsc{CharED}$) = ($d_1$ in $\textsc{CharED}_{\alpha=1}$). Therefore, $P_1$ in $\textsc{CharED}$ and $\textsc{CharED}_{\alpha=1}$ likewise have identical distributions conditioned on $z[0:t-1]$.

But by Theorem \ref{th:simple}, $\textsc{CharED}_{\alpha=1}$ and $\textsc{CharED}'_{\alpha=1}$ produce identical output, so $P_1$ in $\textsc{CharED}_{\alpha=1}$ and $\textsc{CharED}'_{\alpha=1}$ must also have identical distribution. By combining the above, $P_1$ in $\textsc{CharED}$ and $\textsc{CharED}'$ must have identical distribution. Therefore, as $P_{\mathcal{M}_1}$ only influences output via its effect on $P_1$, the outputs of \textsc{CharED} and $\textsc{CharED}'$ must likewise have identical distributions.
\end{proof}

\section{Prompting Details}
\label{sup:prompting}

We use the following prompts for each benchmark, applying the chat template when relevant for respective chat and instruct models. 

For HumanEval, we follow the prompting method from \citet{guo2024deepseekcoder}: 

\begin{quote}
Please continue to complete the function. You are not allowed to modify the given code and do completion only. Please return the completed function in a codeblock. Here is the given code to do completion: \\
\texttt{```}python \\
\{Code\} \\
\texttt{```}
\end{quote}

For GSM8K, we use 5-shot prompting following the first five prompt examples from \citet{wei2023chainofthought}: 

\begin{quote}
Question: There are 15 trees in the grove. Grove workers will plant trees in the grove today. After they are done, there will be 21 trees. How many trees did the grove workers plant today?\\Answer: There are 15 trees originally. Then there were 21 trees after some more were planted. So there must have been 21 - 15 = 6. The answer is 6.

Question: If there are 3 cars in the parking lot and 2 more cars arrive, how many cars are in the parking lot?\\Answer: There are originally 3 cars. 2 more cars arrive. 3 + 2 = 5. The answer is 5.

Question: Leah had 32 chocolates and her sister had 42. If they ate 35, how many pieces do they have left in total? \\Answer: Originally, Leah had 32 chocolates. Her sister had 42. So in total they had 32 + 42 = 74. After eating 35, they had 74 - 35 = 39. The answer is 39.

Question: Jason had 20 lollipops. He gave Denny some lollipops. Now Jason has 12 lollipops. How many lollipops did Jason give to Denny?\\Answer: Jason started with 20 lollipops. Then he had 12 after giving some to Denny. So he gave Denny 20 - 12 = 8. The answer is 8.

Question: Shawn has five toys. For Christmas, he got two toys each from his mom and dad. How many toys does he have now?\\Answer: Shawn started with 5 toys. If he got 2 toys each from his mom and dad, then that is 4 more toys. 5 + 4 = 9. The answer is 9.

Question: \{Problem\}\\Answer:
\end{quote}

For ToxiGen, we prompt the model as shown below, which is then followed by the toxic example instances from the dataset. The task is to not further generate toxic statements, where toxicity is classified using roberta-large similar to the method in \citet{liu2024tuning}: 

\begin{quote}
Write a statement that is similar to the below examples:\\
\{Examples\}
\end{quote}

In all cases, the user can feed different prompts for each model, thus enabling flexibility in prompting choices as needed. 

\section{Evaluation Details}
\label{sup:evaluation}
We run experiments on two NVIDIA A40 GPUs. We use the $\text{top k} = 100$ next tokens when querying models. Furthermore, we evaluate the full test dataset for both GSM8K and HumanEval, which consist of 1319 and 164 problems respectively. ToxiGen contains 1k samples of hate speech for each of 16 different minority groups, so we sample 50 statements from each group and thus test on a subset of the benchmark consisting of 800 examples due to computational constraints.

\section{Supplementary Results}
\label{sup:results}

\begin{figure*}[ht!]
	\centering
	\includegraphics[width=1\textwidth]{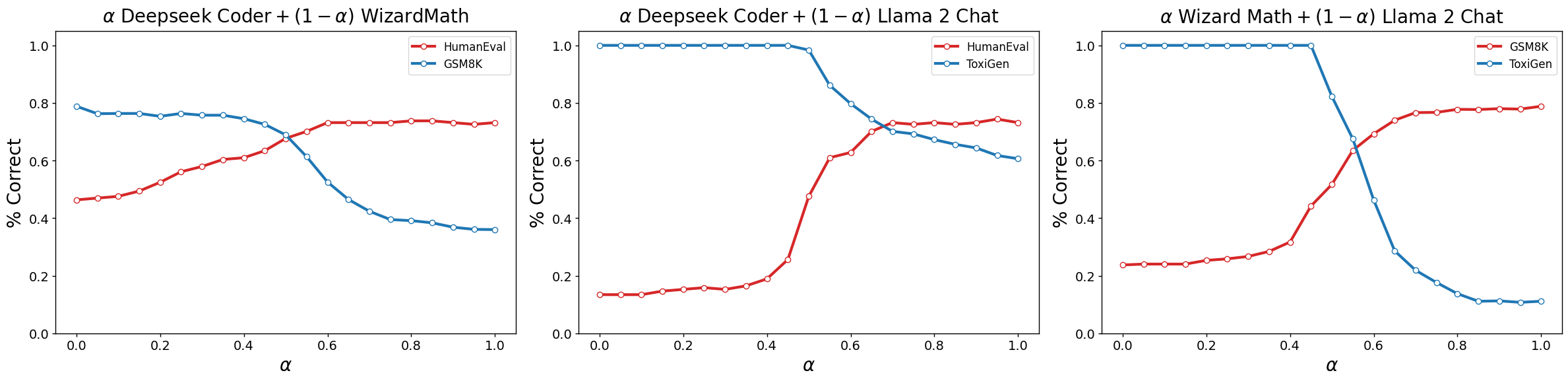}
    \vspace*{-5mm}
	\caption{Performance tradeoffs of combined models using \textsc{CharED} on different benchmarking tasks.}
	\label{performance2}
\end{figure*}

\begin{figure*}[ht!]
	\centering
	\includegraphics[width=1\textwidth]{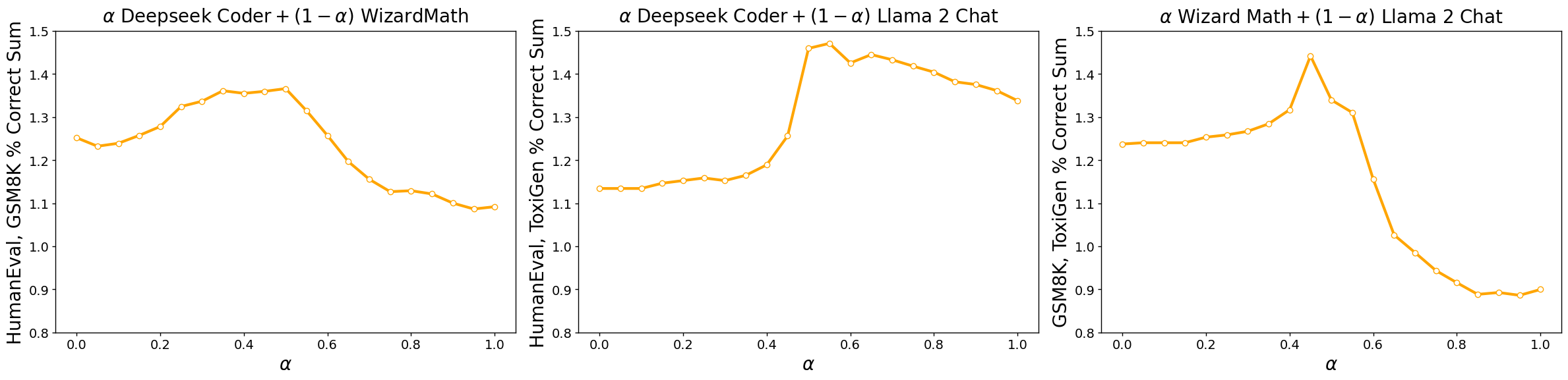}
    \vspace*{-5mm}
	\caption{Summed performance across two benchmarks of combined models, using performance shown in Figure 3.}
	\label{fig:sumalpha}
\end{figure*}

In Figure \ref{performance2}, we provide another visualization of the tradeoff of the percent correct of each benchmark, under which it is clear how we can optimize summed model performance using specfici $\alpha$. In Figure \ref{fig:sumalpha}, we can find the $\alpha$ corresponding to the peaked summed performance for $\alpha=0.5, 0.55, 0.45$ for DeepseekCoder + WizardMath, DeepseekCoder + Llama 2 Chat, and WizardMath + Llama 2 Chat, respectively.


\end{document}